%% file: colt2026-sample.tex
\documentclass[preprint,12pt]{colt2026} 

\newtheorem{assumption}[theorem]{Assumption}
\title[Constrained Contextual Bandits with Adversarial Contexts via Regression]{A Simple Reduction Scheme for Constrained Contextual Bandits with Adversarial Contexts via Regression}
\usepackage{times}

\newcounter{commentcounter}

\newcommand{\Regret}{$\mathsf{Regret}$}
\newcommand{\CCV}{$\mathsf{CCV}$}
\newcommand{\CCB}{$\mathsf{CCB}$}
\newcommand{\IGW}{$\mathsf{IGW}$}
\newcommand{\CB}{$\mathsf{CB}$}
\newcommand{\SqCB}{$\mathsf{SquareCB}$}
\newcommand{\OPT}{$\mathsf{OPT}$}
\newcommand{\cbwk}{$\mathsf{CBwK}$}
\newcommand{\cbwlc}{$\mathsf{CBwLC}$}
\usepackage{bm}
\usepackage{nicefrac} 
\usepackage{algorithm, algorithmic}
\usepackage{enumitem}
\usepackage{booktabs}

\coltauthor{%
 \Name{Dhruv Sarkar} \Email{dhruv.sarkar223@gmail.com}\\
 \addr Indian Institute of Technology Kharagpur, India
 \AND
 \Name{Abhishek Sinha} \Email{abhishek.sinha@tifr.res.in}\\
 \addr Tata Institute of Fundamental Research, Mumbai, India%
}

\begin{document}

\maketitle

\input{abstract}


\input{intro}

\input{problem_formulation}

\input{preliminaries}
\input{results}
\input{conclusion}
\newpage
\bibliography{OCO}

\appendix
\newpage
\input{budget_constraints}

\input{CBwK}
\section{Sharper bound under the Non-negative Expected Regret Assumption}
\label{sec:non-neg-reg}
While the expected regret (also known as \emph{pseudo-regret}) is always non-negative in the stochastic setting, in the constrained setting, the expected regret could be negative. This is because the comparator policy is constrained as it has to satisfy the feasibility condition at every round, while the online policy is permitted a small violation over $T$ rounds (see Eqn.\ \eqref{q-bd-no-assump} for a bound). Nevertheless, as shown below, one can derive a tighter \CCV~ bound under the weaker assumption that the average regret is $-\Theta(\sqrt{T})$. 
If \[\frac{1}{T}\sum_{t=1}^T \mathbb{E} \mathsf{Regret}_t(\pi^\star) \geq - c\sqrt{KTU_T},\]
for some constant $c \geq 0.$
Then from Eqn.\ \eqref{Q-decomp-2}, we have 
\begin{eqnarray*}
		R_T^2  \leq (4+c) VT\sqrt{KU_TT } + T^2 + 8T\sqrt{KU_T}R_T.
\end{eqnarray*}
Solving the above quadratic inequality, we conclude
\begin{eqnarray*}
	 \sqrt{\sum_{\tau=1}^T \mathbb{E}Q^2(\tau)} \equiv R_T = O(T\sqrt{KU_T}). 
\end{eqnarray*}
From Eqn. \eqref{eq:original_ineq} it follows that
\begin{align}
\mathbb{E}Q^2(T) + V\mathbb{E} \textrm{Regret}_T (\pi^\star) \leq 4V\sqrt{K U_TT} + 2T + 8 \sqrt{K U_T}R_T
\end{align}
Finally, choosing the parameter $V = \sqrt{KU_T T}$, we obtain
\begin{align*}
    \mathbb{E} \textrm{Regret}_T (\pi^\star) = O(\sqrt{KU_T T}), ~~    \mathbb{E}Q(T) = O(\sqrt{KU_T T}).
\end{align*}

\section{$\mathsf{CBwLC}$ with Stochastic Contexts}
We now consider stochastic contexts, signed costs, and a long-term budget feasible benchmark, first considered by \citet{slivkins2023contextual}. However, unlike \citet{slivkins2023contextual}, who consider the large-budget regime ($B_T= \Omega(T)$), we consider the complimentary case of small budget regime ($B_T=o(T).$) This setting must be contrasted with the Contextual Bandits with Knapsack constraints ($\mathsf{CBwK}$), which only consider non-negative costs for arbitaray budget and adversarial contexts. 

Continuing from inequality \eqref{eq:with_benchmark} and taking expectation w.r.t. the random context $x_t$, we have
\begin{align}
 	&\mathbb{E}(\Phi(Q(t))|\mathcal{F}_{t-1}) - \Phi(Q(t-1))+ \mathbb{E}_{x_t\sim \mathbb{P}}\langle \bm{f}^\star(x_t), \bm{\pi^\star}(x_t) - \bm{\pi}_t(x_t)\rangle \nonumber \\
 	 &\leq \mathbb{E}_{x_t\sim \mathbb{P}} \langle \bm{L}^\star_t(x_t),\bm{\pi}^\star(x_t) -\bm{\pi}_t(x_t) \rangle  + \frac{1}{2}(\Phi''(Q(t))+ \Phi''(Q(t-1))) + \nonumber \\
     &\Phi'(Q(t-1)) \mathbb{E}_{x_t\sim \mathbb{P}}\mathbb{E}_{a^\star \sim \pi^\star(x_t)} g_t(x_t, a^\star),
 \end{align} 
where $\mathbb{P}$ is the (unknown) law according to which the contexts are distributed. Further note that due to stationarity of the benchmark policy, we have  $\mathbb{E}_{x_t\sim \mathbb{P}}\mathbb{E}_{a^\star \sim \pi^\star(x_t)} g_t(x_t, a^\star) \leq b = \frac{B_T}{T}.$
We now choose the Lyapunov function to be $\Phi(x) = \frac{x^2}{V}$, for some positive parameter $V,$ to be fixed later. Summing up the above inequalities up to round $t$ and then taking expectation over the remaining randomness, we have that for any $t \in [T]:$
\begin{align}
\label{eq:original_ineq_cbwlc}
&\mathbb{E}Q^2(t) - Q^2(0) + V\mathbb{E} \textrm{Regret}_t (\pi^\star) \nonumber \\ \leq &4V\sqrt{K U_Tt} + 2t + 8 \sqrt{K U_T}\sqrt{\sum_{\tau=1}^t \mathbb{E}Q^2(t)} + \frac{B_T}{T}\sum_{\tau=1}^t \mathbb{E}Q(\tau).
\end{align}
Summing the above from $t=1$ to $T$, and noting that $Q(0)=0,$ we have
\begin{align} \label{Q-decomp-2_cbwlc}
    &\sum_{t=1}^T\mathbb{E}Q^2(t) + V\sum_{t=1}^T\mathbb{E} \textrm{Regret}_t (\pi^\star) \nonumber \\ \leq &4VT^{3/2}\sqrt{K U_T} + 2T^2 + 8 T\sqrt{K U_T}\sqrt{\sum_{t=1}^T \mathbb{E}Q^2(t)} + B_T\sum_{t=1}^T \mathbb{E}Q(t).
\end{align}
Applying Cauchy-Schwarz inequality to $\sum_{t=1}^T \mathbb{E}Q(t)$ we get,
\begin{align*}
    \sum_{t=1}^T \mathbb{E}Q(t) \leq \sqrt{T}\sqrt{\sum_{t=1}^T\mathbb{E}Q^2(\tau)}. 
\end{align*}
Let us now define the variable $R(T) := \sqrt{\sum_{t=1}^T \mathbb{E}Q^2(t)}.$ Note that we trivially have $\mathbb{E} \textrm{Regret}_t (\pi^\star) \geq -2T$. This is because $\mathbb{E}\textrm{Regret}_t (\pi^\star) = \sum_\tau f(\pi_{\tau}(x_\tau),x_\tau) - f(\pi^\star(x_\tau),x_\tau)$ and since we assume that the function $f$ takes values in $[-1,1],$ we have $f(\pi_{\tau}(x_\tau),x_\tau) - f(\pi^\star(x_\tau),x_\tau) \geq -2\,\, \forall \tau \in [T]$. Plugging in this bound in inequality \eqref{Q-decomp-2_cbwlc}, we conclude:
\begin{align*}
    R^2(T) \leq VT^2 + 2T^2 + 4VT^{3/2}\sqrt{K U_T} + 8 T\sqrt{K U_T}R(T) + B_T\sqrt{T}R(T).
\end{align*}
Noticing that the above inequality is of the form $x^2 \leq ax + b$ where $x\equiv R(T),$ and using the bound from Lemma \ref{lemma:quadratic}, we get the following upper bound on $R(T)$:
\begin{align}
\label{eq:bound-R(t)_cbwlc}
    R(T) \leq \sqrt{V}T + 2T + 2K^{1/4}\sqrt{V}T^{3/4} (U_T)^{1/4} + 8 T\sqrt{K U_T} + B_T\sqrt{T}.
\end{align}

We further note that inequality \eqref{eq:original_ineq_cbwlc} can be rewritten in term of $R(T)$ as follows:
\begin{align} \label{q-reg-ineq0_cbwlc}
\mathbb{E}Q^2(T) + V\mathbb{E} \textrm{Regret}_T (\pi^\star) \leq 4V\sqrt{K U_TT} + 2T + 8 \sqrt{K U_T}R(T) + \frac{B_T}{\sqrt{T}}R(T).
\end{align}
Plugging in the upper bound on $R(T)$ from \eqref{eq:bound-R(t)_cbwlc} into the above inequality, we obtain
\begin{align} \label{Q-reg-ineq1_cbwlc}
&\mathbb{E}Q^2(T) + V\mathbb{E} \textrm{Regret}_T (\pi^\star) \nonumber \\ \leq &4V\sqrt{K U_TT} + 2T + 8 \sqrt{VK U_T}T + 16\sqrt{KU_T}T \nonumber \\ &+ 16 K^{3/4}\sqrt{V} T^{3/4} U_T^{3/4} + 64KU_TT + B_T^2.
\end{align}
Hence, using $\mathbb{E}(Q^2(T))\geq 0,$ we obtain the following regret bound:
\begin{align*}
    \mathbb{E} \mathsf{Regret}_T \leq O\bigg(\max(\sqrt{K TU_T}, \sqrt{\frac{KU_T}{V}}T, \frac{1}{\sqrt{V}}K^{3/4} T^{3/4} U_T^{3/4}, \frac{KU_T T}{V}, \frac{B_T^2}{V})\bigg).
\end{align*}
Furthermore, substituting the trivial regret lower bound $\mathbb{E} \textrm{Regret}_t (\pi^\star) \geq -2T$ into \eqref{Q-reg-ineq1_cbwlc} and using Jensen's inequality $\mathbb{E}Q^2(T) \geq (\mathbb{E}(Q(T)))^2,$ we obtain the following \CCV~bound
\begin{eqnarray} \label{q-bd-no-assump_cbwlc}
    \mathbb{E}Q(T)\leq O\bigg(\max(\sqrt{VT},\sqrt{V}(K TU_T)^{1/4},  (VK U_T)^{1/4}\sqrt{T} , K^{3/8}V^{1/4} T^{3/8} U_T^{3/8}, \sqrt{KU_TT}, B_T)\bigg).
\end{eqnarray}
Choosing $V = \sqrt{KTU_T} + B_T,$ we conclude $ \mathbb{E} \mathsf{Regret}_T = O(\sqrt{K}T^{3/4}U_T^{1/4}+B_T), ~\mathbb{E}\mathsf{CCV}_T = O(\sqrt{K}T^{3/4}U_T^{1/4} + \sqrt{TB_T}),$ where we have used the fact that $B_T \leq T.$
\section{Auxilliary Lemmas}
\begin{lemma}
\label{lemma:quadratic}
    If $x^2 \leq ax + b$ with $a,b \geq 0$ then it implies that $x \leq a + \sqrt{b}$.
\end{lemma}
\begin{proof}
    \begin{align*}
        x^2 \leq ax + b \implies \bigg(x-\frac{a}{2}\bigg)^2 \leq \frac{a^2}{4} + b \implies x-\frac{a}{2} \leq \sqrt{\frac{a^2}{4} + b} \implies x \leq \frac{a}{2} + \sqrt{\frac{a^2}{4} + b}
    \end{align*}

Further note that,
\begin{align*}
    \sqrt{\frac{a^2}{4} + b} \leq \frac{a}{2} + \sqrt{b}
\end{align*}
as $\sqrt{x + y} \leq \sqrt{x} + \sqrt{y},$ for $x,y \geq 0$.
Putting everything together, we get,
\begin{align*}
    x^2 \leq ax + b \implies x \leq a + \sqrt{b}.
\end{align*}
\end{proof}
\end{document}

%% file: abstract.tex
\begin{abstract}
We study constrained contextual bandits (\CCB) with adversarially chosen contexts, where each action yields a random reward and incurs a random cost. We adopt the standard realizability assumption: conditioned on the observed context, rewards and costs are drawn independently from fixed distributions whose expectations belong to known function classes. We consider the continuing setting, in which the algorithm operates over the entire horizon even after the budget is exhausted. In this setting, the objective is to simultaneously control regret and cumulative constraint violation. Building on the seminal $\mathsf{SquareCB}$ framework of \citet{foster2018practical}, we propose a simple and modular algorithmic scheme that leverages online regression oracles to reduce the constrained problem to a standard unconstrained contextual bandit problem with adaptively defined surrogate reward functions. In contrast to most prior work on \CCB, which focuses on stochastic contexts, our reduction yields improved guarantees for the more general adversarial context setting, together with a compact and transparent analysis. 
\end{abstract}

%% file: intro.tex
\section{Introduction}
Contextual Bandit ($\mathsf{CB}$) is a standard online learning framework for studying decision-making under uncertainty with side information. They can be viewed as a natural generalization of the classical Multi-Armed Bandit (MAB) problem where some additional side information conveying an implicit actions-to-rewards mapping is made available to the learner in the form of \emph{contexts}. \CB~has found widespread applications in practice, including personalized recommendation systems, crowdsourcing platforms, and large-scale clinical trials. In a typical setting, the learner interacts with the environment over multiple rounds; in each round, the learner observes a context, selects an action, and subsequently receives a reward. The goal of the learner is to maximize the cumulative reward over a given horizon by learning an optimal context-to-action mapping that appropriately balances exploration with exploitation. This problem has been extensively studied in the literature, and a variety of practical and efficient algorithms with near-optimal performance guarantees have been developed \citep{beygelzimer2010optimal,agarwal2014taming,agarwal2016making}.  

Although contextual bandits are widely used in practice, many applications cannot be adequately captured by a single scalar performance metric. Policies are often subject to additional operational limitations in the form of long-term constraints, which are not modeled within the vanilla \CB~framework. For instance, recommendation systems must optimize ad placements while respecting advertisers’ budget constraints. In crowdsourcing applications, task allocation must be performed while ensuring fairness and compliance with labour regulations. In clinical trials, clinicians are required to design patient-specific treatment plans while managing limited medical resources \citep{tewari2017ads, bouneffouf2019survey}. Similarly, large language models must generate useful responses to user queries while adhering to safety and alignment requirements \citep{monea2024llms}.

To address these additional requirements, the study of constrained contextual bandits ($\mathsf{CCB}$) has recently been initiated, giving rise to a growing body of literature that already contains a number of compelling results \citep{badanidiyuru2014resourceful, agrawal2014bandits}. In this problem, the objective is to maximize the cumulative rewards while satisfying a long-term cost constraint. To maintain analytical tractability, existing works on \CCB ~typically assume that contexts are independent and stationary, and are sampled in an i.i.d.\ manner from an underlying distribution. However, in many practical applications, successive contexts may be correlated and may exhibit distributional shift (\emph{e.g.,} in recommender systems, user preferences may vary over the course of a day, or in clinical trials, patient demographics and case histories may evolve over time).

A standard approach to addressing these challenges is to adopt a more robust adversarial model, which places no assumptions on the context-generating process and therefore remains effective even when the environment changes its behaviour over time. Such an adversarial formulation also enables the study of settings in which contexts are adaptively chosen based on the learner’s past actions, which arises in competitive environments such as online auctions or security games \citep{harris2024regret,hu2025learning}.

However, extending contextual bandits from stochastic to adversarial contexts introduces significant technical challenges in both algorithm design and analysis \citep[Remark 2.2]{slivkins2023contextual}. To recall the state of the art, note that \citet{guo2024stochastic}, which is closely related to this work, relies heavily on a multi-step Lyapunov drift analysis to establish high-probability bounds on the virtual queue by modelling the queue evolution as a Markovian process. This approach, however, is inherently limited by the stochastic assumption, which enables the authors to amortize estimation errors over multiple rounds and thereby ensure the stability of the virtual queues. More recent works, such as \citet{guo2025stochastic}, make an even stronger assumption that the context distribution is known \emph{a priori}, which is then leveraged to derive tighter performance guarantees. Two concurrent works by \citet{han2023optimal} and \citet{slivkins2023contextual} consider the Contextual Bandit with Long-Term Budget Constraints. \citet{slivkins2023contextual} considers both positive and negative resource consumption (which they call the $\mathsf{CBwLC}$ problem), where they assume a known and strictly positive Slater constant and a large budget regime ($B_T = \Omega(T)$). \citet{han2023optimal} consider the case with hard stopping, but they only consider the case where $B_T = \Omega(T^{3/4})$.

In this paper, we study the \CCB~problem with adversarial contexts, directly building upon the unconstrained \SqCB~framework of \citet{foster2020beyond} to constrained settings with long-term constraints. Our algorithm uses an off-the-shelf online regression-based oracle under the realizability assumption \citep{foster2018practical}. Specifically, similar to \SqCB, we employ the online regression oracle to maintain running estimates of the mean reward and cost functions based on past observations. These estimates are then used to construct a surrogate objective, which is minimized via an inverse-gap-weighting (\IGW) policy with an adaptive learning rate (see Figure~\ref{flowchart} for a schematic). The resulting \IGW~policy carefully balances exploration, exploitation, and constraint satisfaction by building upon a regret decomposition scheme first introduced in \citet{sinha2024optimal} in the context of online convex optimization with long-term constraints.

Overall, our approach, given by Algorithm \ref{ccb2}, yields a unified, modular, and flexible algorithmic framework that enables a streamlined analysis of \CCB~through a single key inequality and offers improved bounds over the state-of-the-art. Our main contributions can be summarized as follows:

\begin{itemize}
	\item We consider constrained contextual bandits with adversarial contexts and offer improved performance bounds under various feasibility assumptions of the benchmark. For example, when we assume \textit{almost sure feasibility} with general costs or expected feasibility with non-negative costs, we achieve \Regret~and \CCV~bounds of the order $\Tilde{O}(\sqrt{TU_T})$ where $U_T$ is the cumulative squared loss of the regression oracle.  Previous works achieved bounds of the order $\Tilde{O}(T^{3/4}U_T^{1/4})$ assuming stochastic contexts, while the area of \CCB~with adversarial contexts remained underexplored. We also consider the case where strict feasibility in the form of Slater's constant $\epsilon$ is considered and provide some valuable insight on a conjecture by \citet{guo2024stochastic}. See Theorem \ref{thm:main} for the precise results.
	\item Our reduction scheme allows us to consider the long-term budget-constrained setting, \emph{a.k.a.} contextual bandits with knapsacks ($\mathsf{CBwK}$). We consider the \textit{continuing} case, where the learner incurs violations even after the budget is exhausted in contrast to the \textit{hard stopping} setting where the learner stops the moment the budget is exhausted. Importantly, we eliminate many of the restrictive assumptions in prior work, including \emph{e.g.,} not assuming strict feasibility in the form of Slater's condition or restricting to the large budget regime $B_T = \Omega(T)$ \citep{slivkins2023contextual}.  We also consider the case of contextual bandits with linear constraints $\mathsf{CBwLC}$ in the stochastic contexts regime, where, unlike $\mathsf{CBwK}$, the costs are not assumed to be non-negative.  In both settings, our results are complementary to those in \citet{slivkins2023contextual} as we provide non-trivial results in the small-budget regime where $B_T=o(T)$. 
    
	\item On the conceptual side, we give a streamlined treatment to the \CCB~problem by reducing it to an instance of the unconstrained contextual bandit problem with a regression oracle. In particular, we derive a general regret decomposition inequality, Eqn.\ \eqref{reg-decomp-ultimate}, making use of the fundamental $\mathsf{SquareCB}$ framework. All our \Regret~and \CCV~bounds under various settings follow from this single inequality after a few lines of algebra. 
\end{itemize}
We refer the reader to Section \ref{comparison} in the Appendix for a comprehensive comparison of our results with the state-of-the-art.

%% file: problem_formulation.tex
\section{Problem Formulation}
 We consider a budget-constrained contextual bandit problem as described next. 
At the beginning of each round $t \geq 1$, the learner observes a context $x_t \in \mathcal{X},$ where $\mathcal{X}$ is the set of all possible contexts. The contexts could be chosen adversarially at each round. Upon observing the context $x_t$, the learner selects an action $a_t \in [K]$ from the set of $K$ possible actions, also referred to as \emph{arms}. The action $a_t$ could be randomized. Subsequently, the learner receives a random reward $f_t(x_t, a_t) \in [-1, 1]$ and incurs a random cost $g_t(x_t, a_t) \in [-1, 1]$.  Given the context $x_t$, the rewards and costs are assumed to be drawn independently from a distribution whose expected values are characterized below. 
\begin{assumption}[Realizability]
\label{assum:realizability}
Let $\mathcal{F}$ and $\mathcal{G}$ be two predefined function classes comprising functions that map each ($\mathsf{context}$, $\mathsf{action}$) tuple to the interval $[-1,1].$  Then the realizability assumption states that there exist functions $f^\star \in \mathcal{F}$ and $g^\star \in \mathcal{G}$ such that $\mathbb{E}[f_t(x_t, a)|x_t=x] = f^\star(x,a)$ and $\mathbb{E}[g_t(x_t, a)|x_t=x] = g^\star(x,a), \forall x \in \mathcal{X}, a \in [K], t\geq 1$. 
\end{assumption}
$\mathcal{F}$ and $\mathcal{G}$ could be user-specified general function classes that may be flexibly implemented with, \emph{e.g.,} decision trees, kernels, neural nets, etc. 
To keep the exposition notationally cleaner, we only consider the exact realizability condition as defined above. Extension of the model misspecification leading to approximate realizability will be straightforward in the sequel. The functions $f^\star$ and $g^\star$ are not known \emph{a priori} and must be learned through past experience. A (randomized) policy $\pi: \mathcal{X} \to \Delta([K])$ is a mapping which maps each context to a probability distribution over the actions. Let $\Pi$ denote the set of all policies. The goal of the learner is to perform \emph{close} to an optimal offline stationary benchmark policy $\pi^\star$ that maximizes the cumulative rewards while satisfying the cost constraints at each round, \emph{i.e.,}
\begin{eqnarray} \label{problem_statement}
   \pi^\star = \arg \max_{\pi \in \Pi} \quad & \sum_{t=1}^T \mathbb{E}_\pi\big(f_t(x_t, a_t)\big), ~~~~
    \text{s.t.} \quad &  \mathbb{E}_\pi \big(g_t(x_t, a_t)\big) \leq 0, ~~ 1\leq t \leq T.
\end{eqnarray}
See Section \ref{metrics} for precise definitions of the performance metrics for an online policy.
In the above, we have insisted that the benchmark $\pi^\star$ satisfies the constraints at \emph{every} round while the online policy (approximately) satisfies the budget constraint in the long-term. This is the standard assumption in the constrained learning literature \citep{guo2025stochastic, slivkins2022efficient, sinha2024optimal}. In Section \ref{sec:lt_budget}, we will consider a special case with non-negative costs where the benchmark $\pi^\star$ is only required to satisfy a given budget constraint $B_T$ over the entire horizon of length $T$, \emph{i.e.,}
\begin{eqnarray} \label{problem_statement2}
	 \pi^\star=  \arg\max_{\pi \in \Pi} \quad & \sum_{t=1}^T \mathbb{E}_\pi\big(f_t(x_t, a_t)\big), ~~~~
    \text{s.t.} \quad &  \sum_{t=1}^T \mathbb{E}_\pi \big(g_t(x_t, a_t)\big) \leq B_T, 
\end{eqnarray}
See Section \ref{benchmarks} for a detailed discussion on the benchmarks.
\subsection{Online Regression Oracle ($\mathcal{O}_{\textrm{sq}}$)}

The learner interacts with the function classes $\mathcal{F}$ and $\mathcal{G}$ through an off-the-shelf online regression oracle $\mathcal{O}_{\mathrm{sq}}$ over $T$ rounds. At each round $t \geq 1$, the oracle takes the context $x_t$ as input and produces predictions for the reward and cost associated with each action. We denote the predicted vectors by $\big(\hat{f}_t(x_t, a), a \in [K]\big)$ and $\big(\hat{g}_t(x_t, a), a \in [K]\big)$, respectively. Let $a_t \in [K]$ denote the arm selected by any (possibly randomized) policy at round $t$, which results in a random reward $f_t(x_t, a_t)$ and a random cost $g_t(x_t, a_t)$. Recall that under the realizability assumption, we have for all $(x_t, a_t):$
\begin{eqnarray} \label{realizability-implies}
\mathbb{E}\!\left[f_t(x_t, a_t)|x_t, a_t\right] = f^\star(x_t, a_t)
\quad \text{and} \quad
\mathbb{E}\!\left[g_t(x_t, a_t)|x_t, a_t\right] = g^\star(x_t, a_t).
\end{eqnarray}
The quality of the predictions produced by $\mathcal{O}_{\mathrm{sq}}$, measured in terms of squared loss, is assumed to satisfy the following guarantees \citep[Definition 3]{foster2023foundations}:
\begin{equation} \label{oracle-guarantee}
\mathbb{E} \sum_{t=1}^T \big(\hat{f}_t(x_t, a_t) - f^\star(x_t, a_t)\big)^2 \leq U_T,
\qquad
\mathbb{E} \sum_{t=1}^T \big(\hat{g}_t(x_t, a_t) - g^\star(x_t, a_t)\big)^2 \leq U_T,
\end{equation}
where the error bound $U_T$ grows sub-linearly with $T$\footnote{Note that although at each round, the oracle produces estimates for all actions, in \eqref{oracle-guarantee}, its quality is measured only with respect to the action of the policy $a_t,$ which may, in turn, depend on the estimated values of all actions. }. In \eqref{oracle-guarantee}, the expectations are taken with respect to both the internal randomness of the online regression oracle $\mathcal{O}_{\mathrm{sq}}$ and the randomness of the policy (which may depend on the oracle outputs $\hat{f}_t$ and $\hat{g}_t$). Under Assumption \ref{assum:realizability}, the bounds in \eqref{oracle-guarantee} can be achieved by any no-regret online learning algorithm competing against the respective function classes under the squared loss function; see \citep[Lemma~6]{foster2023foundations}. The value of $U_T$ depends on the complexity of the function classes. For example, if both $\mathcal{F}$ and $\mathcal{G}$ are finite and the oracle is implemented using the Exponential Weights algorithm, then one may take
\[
U_T = O\!\left(\log\!\big(\max\{|\mathcal{F}|, |\mathcal{G}|\}\big)\right),
\]
as shown in \citet[Proposition~3]{foster2023foundations}. Similarly, for $d$-dimensional linear function classes, using the classic Vovk--Azoury--Warmuth forecaster or Online Newton Step (ONS) yields
\[
U_T = O(d \log T)
\]
under standard regularity conditions \citep[Theorem~7.34]{orabona2019modern}.

In practice, the regression oracle may be implemented using an artificial neural network trained with online gradient descent. Throughout the remainder of the paper, we treat $\mathcal{O}_{\mathrm{sq}}$ as a given black box and focus on designing the online learner that invokes this oracle.

\subsection{Offline Benchmarks for Regret Computation} \label{benchmarks}
In this paper, we consider several classes of offline benchmark policies used to gauge the performance of an online policy. These benchmarks differ in how they enforce the budget constraints, and this has been considered in previous work. In this paper, we provide a unified treatment through a regret decomposition inequality.  
\begin{definition}[Feasible in Expectation]
\label{assum:in_expect}
  A stationary policy $\pi^\star : \mathcal{X} \to \Delta_K$ is called \emph{feasible in expectation} if $\pi^\star$ yields non-positive cost in expectation every round, \emph{i.e.,} 
\begin{eqnarray*}
	\mathbb{E}_{a \sim \pi^\star(x_t)} g_t(x_t, a) \leq 0, ~~ \forall t.
\end{eqnarray*}
\end{definition}

\begin{definition}[Feasible in Expectation with Slater's condition]
\label{assum:in_expect-slater}
  A stationary policy $\pi^\star : \mathcal{X} \to \Delta_K$ is called \emph{feasible in expectation} with Slater parameter $\epsilon>0$ if $\pi^\star$ yields non-positive cost in expectation every round with an $\epsilon$ slack, \emph{i.e.,} 
\begin{eqnarray*}
	\mathbb{E}_{a \sim \pi^\star(x_t)} g_t(x_t, a) \leq -\epsilon, ~~ \forall t.
\end{eqnarray*}
\end{definition}

\begin{definition}[Almost Surely Feasible]
\label{assum:almost_sure}
    A stationary policy $\pi^\star : \mathcal{X} \to \Delta_K$ is called \emph{almost surely feasible} if $\pi^\star$ yields non-positive cost almost surely every round, \emph{i.e.,} 
\begin{eqnarray*}
	\pi^\star(x_t, a) >0 \implies g_t(x_t, a) \leq 0, ~~\forall a, t. 
\end{eqnarray*}
\end{definition}
\begin{definition}[Long-term Budget Feasible]
\label{assum:lt_budget}
    A stationary policy $\pi^\star : \mathcal{X} \to \Delta_K$ is called \emph{long-term budget feasible} for a total budget of $B_T \geq 0$ if 
\begin{eqnarray*}
	\sum_{t=1}^T \mathbb{E}_{a \sim \pi^\star(x_t)} g_t(x_t, a) \leq B_T.
\end{eqnarray*}
In other words, a long-term budget-feasible benchmark satisfies a given cost constraint in expectation over the entire horizon. 
\end{definition}

\begin{remark}
    Clearly, the relative strengths of the benchmarks are related as follows:
    Long-term Budget Feasible $\supseteq$ Feasible in Expectation $\supseteq$ Almost surely Feasible. We also have Feasible in Expectation $\supseteq$ Feasible in Expectation with Slater's condition. We will see that, as expected, relatively weaker benchmarks lead to stronger performance bounds. 
\end{remark}



\subsection{Performance Metrics} \label{metrics} 
Since any online policy must learn the reward and cost functions sequentially through past experiences, it will necessarily be sub-optimal compared to an offline benchmark $\pi^\star$ that knows $f^\star$ and $g^\star$ \emph{a priori}. The sub-optimality gap of any online policy in terms of the cumulative reward and constraint satisfaction is captured by two quantities, \Regret ~and \CCV, as defined next.  
\paragraph{Regret:} We compare the expected cumulative reward obtained by any online policy with that of the best fixed policy in hindsight that solves the offline problem \eqref{problem_statement} for a given context sequence $x_{1:T}$. 
Specifically, 
let \OPT~denote the cumulative reward accrued by the benchmark. 
The regret of any online policy that takes the (randomized) actions $\{a_t\}_{t=1}^T$ is defined as:
\begin{align}
    \mathsf{Regret}_T := \mathsf{OPT} - \left[\sum_{t=1}^T f_t(x_t, a_t)\right]. 
\end{align}

\paragraph{Cumulative Constraint Violation (\CCV):} An online policy need not satisfy the constraint at every round as the expected cost function is unknown. 
The cumulative constraint violation (\CCV) of any policy that takes the actions $\{a_t\}_{t=1}^T$ is defined as follows:
\begin{align}
    \mathsf{CCV}_T := \left[\sum_{t=1}^T g_t(x_t, a_t)\right]. 
\end{align}
Positive values of \CCV~capture the extent to which the constraints are violated by the online policy in the long run. Clearly, \Regret~and~\CCV~are random variables with the randomness coming from the rewards, costs, and the randomness of the policy.  Our objective is to design online policies that minimize expectation of \Regret~and \CCV~simultaneously. 

%% file: preliminaries.tex
\section{Preliminaries} \label{algo_analysis}
We extend the \SqCB~algorithm of \citet{foster2020beyond}, originally developed for unconstrained contextual bandits, to the constrained contextual bandit setting in a black-box manner. The (loss version of) vanilla \SqCB~subroutine is summarized in Algorithm~\ref{ccb} (reward is negative loss). It employs an online regression oracle $\mathcal{O}_{\text{sq}}$ to estimate each arm's losses from observed contexts, and then feeds these estimates into the classic Inverse Gap Weighting (IGW) policy. The IGW policy, described in Definition \ref{igw-def}, is designed to carefully balance exploration, exploitation, and estimation error, thereby achieving favourable regret guarantees.

\begin{algorithm}[H]
\caption{\SqCB: Contextual Bandits with Regression Oracles}
\label{ccb}
\begin{algorithmic}
  \REQUIRE Online regression oracle $\mathcal{O}_{\text{sq}}$ and parameter $\gamma > 0$

  \FOR{$t = 1, \dots, T$}
    \STATE Receive context $x_t$.
    \STATE Ask $\mathcal{O}_{\text{sq}}$ to predict the loss for each action,
           obtaining $\widehat{y}_t(1), \dots, \widehat{y}_t(K)$.
    \STATE Compute $p_t \in \Delta(K)$ as follows:
    \STATE \hspace{3\algorithmicindent}
      $p_t(a) \gets
        \dfrac{1}{\lambda + 2\gamma\bigl(\widehat{y}_t(a)-\min_b \widehat{y}_t(b)\bigr)}.$

    \STATE Sample $a_t \sim p_t$, observe $\ell_t(a_t)$, and feed
           $((x_t, a_t), \ell_t(a_t))$ to $\mathcal{O}_{\text{sq}}$.
  \ENDFOR
\end{algorithmic}
\end{algorithm}

\begin{figure*}
\centering
	\includegraphics[scale=0.7]{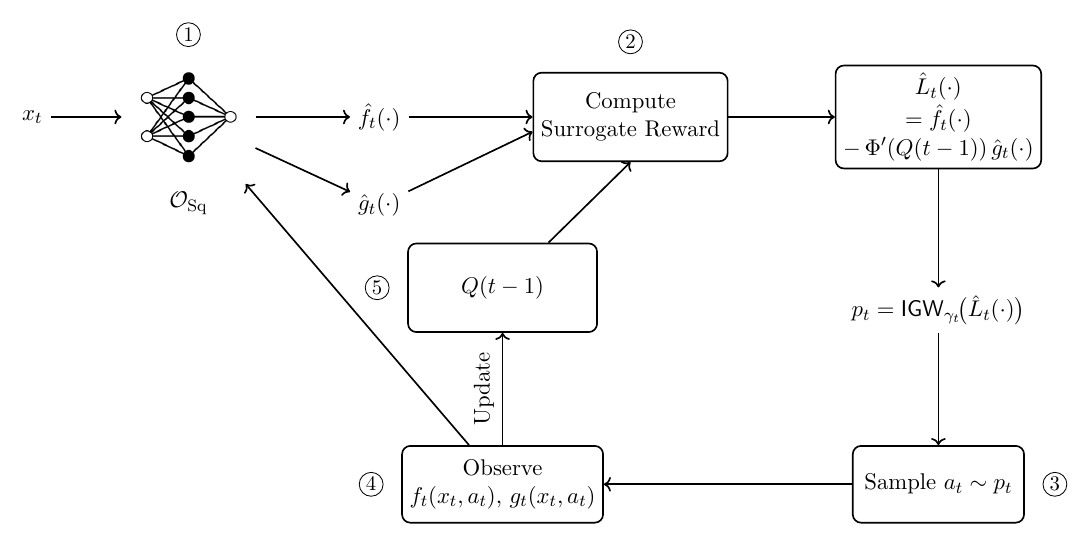}
	\caption{\small{A schematic of the proposed algorithmic scheme for the constrained contextual bandit (\CCB) problem. The numbers within the circles show the sequence of operations performed at any round $t \geq 1$. The variable $Q(t)$ denotes the \CCV~after round $t$ and \IGW~denotes inverse gap weighting.} 
 }
	\label{flowchart}
\end{figure*} 
\begin{definition}[\textsc{Inverse Gap Weighting} \citep{foster2020beyond}] \label{igw-def}
	Given any vector $\hat{\bm{v}} \in \mathbb{R}^K,$ the  \textsc{Inverse Gap Weighting} distribution $p = \mathsf{IGW}_\gamma(\hat{\bm{v}})$ with parameter $\gamma \geq 0$ is defined as
	\[p(a) =  \frac{1}{\lambda + 2 \gamma (\hat{v}(a) - \hat{v}(a^\star))}, ~~ a \in [K],\]
	where $a^\star = \arg \min_{a \in [K]}\hat{v}(a)$ is the greedy action, and $\lambda \in [1, K]$ is chosen such that $\sum_a p(a)=1.$
\end{definition}
The following algebraic result plays a central role in the analysis of the \textsf{SquareCB} algorithm and serves as the starting point for the design and analysis of the \CCB~algorithm presented in this paper.

\begin{lemma} \label{SqCBthm}
	Fix any arbitrary $\hat{\bm{v}} \in \mathbb{R}^K$  and the parameter $\gamma >0$. Then, for the probability distribution $p =\textsf{IGW}_{\gamma}(\hat{\bm{v}})$ given in Definition \ref{igw-def}, it holds that for any vector $\bm{v} \in \mathbb{R}^K$ and any distribution $\bm{\mu} \in \Delta_{K},$ we have 
	\begin{eqnarray} \label{SqCB}
		\langle \bm{v}, \bm{p}\rangle - \langle \bm{v}, \bm{\mu}\rangle  \leq \frac{K}{2\gamma}+ \gamma\sum_{a=1}^K p(a) (\hat{v}(a)- v(a))^2 = \frac{K}{2\gamma} + \gamma\mathbb{E}_{a \sim \bm{p}}(v(a)-\hat{v}(a))^2. 
	\end{eqnarray}
\end{lemma}

In the online learning setting, the LHS of \eqref{SqCB} may be interpreted as the incremental regret for learning the cost vector $\bm{v}$ and the second term on the RHS may be interpreted as the expected estimation error while using the \IGW~policy. See \citet[Proposition 9]{foster2023foundations} for a proof of Lemma \ref{SqCBthm}. 

\section{Constrained Contextual Bandits via the Regret Decomposition Inequality} 
In this section, we present a simple reduction scheme, given by Algorithm \ref{ccb2}, that directly invokes the \SqCB~algorithm, described in Algorithm~\ref{ccb}, to address the constrained contextual bandit problem. In contrast to the unconstrained setting, where the oracle is trained solely on the loss functions, in the constrained setting, we train the online oracle using an adaptively defined \emph{surrogate reward function}. This surrogate reward linearly combines the reward and cost functions using coefficients that are chosen adaptively over time, as described in detail below. A schematic illustration of the overall algorithmic framework is given in Figure~\ref{flowchart}. 
\begin{algorithm}[H]
\caption{Constrained Contextual Bandits with Regression Oracles}
\label{ccb2}
\begin{algorithmic}
 \REQUIRE A non-decreasing convex Lyapunov function $\Phi(\cdot)$
 \FOR{$t = 1, \dots, T$}
\STATE Invoke Algorithm \ref{ccb} with $\hat{y}_t = \hat{L}_t$ given by Eqn. \eqref{surr-cost-def} with the parameter $\gamma_t$ given by Eqn.\ \eqref{gamma-param}.  
\ENDFOR
\end{algorithmic}
\end{algorithm}

While our approach bears some resemblance to the LOE2D framework of \citet{guo2024stochastic}, it differs in several technical aspects, leading to stronger theoretical guarantees under significantly weaker assumptions and a compact, straightforward proof. From a purely algorithmic point of view, while the exploration parameter $\gamma_t$ of the \IGW~policy in this paper depends adaptively on all previous \CCV~variables (see Eqn.\ \eqref{gamma-param} below), the corresponding parameter in \citet{guo2024stochastic} depends only on the current \CCV~(see the $\beta_t$ update in \citet[Eqn (5)]{guo2024stochastic}.) This adaptive choice of the parameters leads to the regret decomposition inequality \eqref{reg-decomp-ultimate}, which, in turn, forms the starting point of our analysis for the adversarial contexts while improving upon the bounds of \citet{guo2024stochastic}. 

\paragraph{Derivation:}
Define the cumulative cost up to round $t$ by $Q(t),$ which satisfies the following recursion:  
\begin{eqnarray} \label{q-recur}
	Q(t) = Q(t-1) + g_t(x_t, a_t).
\end{eqnarray}
 
Let $\Phi: \mathbb{R}_+ \to \mathbb{R}_+$ be a twice differentiable, convex Lyapunov function. We assume $\Phi''(\cdot)$ to be monotone. Using a second-order Taylor series expansion for $\Phi$, we have:  
\begin{eqnarray*}
	\Phi(Q(t)) = \Phi(Q(t-1)) + \Phi'(Q(t-1)) g_t(x_t, a_t) + \frac{1}{2}\Phi''(\zeta) c_t^2,
	\end{eqnarray*}
	for some $\zeta$ that lies between $Q(t-1)$ and $Q(t).$ 
	Using the monotonicity and non-negativity of $\Phi''(\cdot),$ and the fact that $g_t^2 \leq 1,$ we can bound the increase in the Lyapunov function on round $t$ as 
	\begin{eqnarray} \label{lyapunov-incr}
		\Phi(Q(t)) - \Phi(Q(t-1)) \leq \Phi'(Q(t-1)) g_t(x_t, a_t) + \frac{1}{2}(\Phi''(Q(t))+ \Phi''(Q(t-1))),
	\end{eqnarray}
    where in the last line, we have used the monotonicity of $\Phi'', $ which leads to the bound $\Phi''(\zeta) \leq \max (\Phi''(Q(t)), \Phi''(Q(t-1))) \leq \Phi''(Q(t))  + \Phi''(Q(t-1)).$

Next, adding $f_t(x_t, a^\star) - f_t(x_t, a_t)$ to both sides of inequality \eqref{lyapunov-incr}, where $a^\star$ is a randomized action following an arbitrary stationary policy $\pi^\star : \mathcal{X} \mapsto \Delta_K$ 
\begin{eqnarray} \label{drift-ineq1}
	\Phi(Q(t)) - \Phi(Q(t-1))+ f_t(x_t, a^\star) - f_t(x_t, a_t)  &\leq& f_t(x_t, a^\star) - \Phi'(Q(t-1)) g_t(x_t, a^\star)\nonumber\\
	&& - \bigg(f_t(x_t, a_t) - \Phi'(Q(t-1)) g_t(x_t, a_t)\bigg)\nonumber\\
	&& \frac{1}{2}(\Phi''(Q(t))+ \Phi''(Q(t-1))) \nonumber \\
	&&+ \Phi'(Q(t-1)) g_t(x_t, a^\star),
\end{eqnarray}
 
 
 \subsection{Surrogate Reward Functions}
 
 Let $\{\mathcal{F}_\tau\}_{\tau \geq 1}$ denote the natural filtration of the sequence of random variables observed before making decisions on each round, \emph{i.e.,} $\mathcal{F}_{t-1} = \sigma\big(\{x_{\tau+1}, f_{\tau}, g_\tau, a_{\tau}\}_{\tau=1}^{t-1}\big), t> 1$. Choosing the benchmark policy $\pi^\star$ to be any feasible in expectation policy (Definition \ref{assum:in_expect}), and taking the conditional expectation (given $\mathcal{F}_{t-1}$) of both sides of Eqn.\ \eqref{drift-ineq1} with respect to the randomness of the reward and cost functions and the randomness of the online and the benchmark policies, it follows that
  \begin{eqnarray} \label{reg-decomp-ineq1}
 	\mathbb{E}(\Phi(Q(t))|\mathcal{F}_{t-1}) - \Phi(Q(t-1))+ \langle \bm{f}^\star(x_t), \bm{\pi^\star}(x_t) - \bm{\pi}_t(x_t)\rangle \nonumber \\
    \leq \langle \bm{L}^\star_t(x_t),\bm{\pi}^\star(x_t) -\bm{\pi}_t(x_t) \rangle 
 	 + \frac{1}{2}(\Phi''(Q(t))+ \Phi''(Q(t-1))),
 \end{eqnarray} 
 where, in the above, we have defined the \emph{target surrogate} function as:
\begin{eqnarray} \label{est-surr-cost}
	L^\star_t(x_t, a) = f^\star(x_t,a) - \Phi'(Q(t-1)) g^\star(x_t,a), ~~ a \in [K], 
\end{eqnarray}
and the \emph{estimated surrogate} function $L_t(x_t, \cdot): [K] \to \mathbb{R}$ as:
\begin{eqnarray} \label{surr-cost-def}
	\hat{L}_t(x_t, a) = \hat{f}_t(x_t,a) - \Phi'(Q(t-1)) \hat{g}_t(x_t,a), ~~ a \in [K]. 
\end{eqnarray}
In Eqn.\ \eqref{drift-ineq1}, we have used the feasibility of the policy $\pi^\star$ which implies that $\mathbb{E}_{a^\star \sim \pi^\star} g_t(x_t,a^\star) \leq 0$.

\subsection{Bounding the \Regret~of the Surrogate Reward functions} \label{surr-reg-bd}

Using Lemma \ref{SqCBthm}, the one-step regret of the surrogate reward can be further upper-bounded as:
\begin{eqnarray} \label{reg-bd-meth2-1}
	 && \langle \bm{L}_t^\star(x_t), \bm{\pi}^\star(x_t) \rangle - \langle \bm{L}_t^\star(x_t), \bm{\pi}_t(x_t) \rangle  \nonumber\\
	  &\stackrel{(a)}{\leq} & \frac{K}{2\gamma_t} + 2\gamma_t \bigg(\mathbb{E}_{a \sim \bm{\pi}_t(x_t)} (f^\star(x_t,a)-\hat{f}_t(x_t,a))^2 + z_t\mathbb{E}_{a \sim \bm{\pi}_t(x_t)} (g^\star(x_t,a)-\hat{g}_t(x_t,a))^2\bigg), \nonumber \\
\end{eqnarray}
where we define $z_{t} \equiv \max\big(1, \big(\Phi'(Q(t-1))\big)^2\big), ~ t\geq 1.$ Next, we choose the parameter $\gamma_t$ as 
\begin{eqnarray} \label{gamma-param}
	\gamma_t = \frac{1}{2z_t} \sqrt{\frac{K}{U_T}\sum_{\tau=1}^{t} z_\tau}, ~~ t \geq 1.
\end{eqnarray}
With this choice, the RHS of \eqref{reg-bd-meth2-1} simplifies to
\begin{eqnarray} \label{reg-decomp-bd}
	&&\langle \bm{L}_t^\star(x_t), \bm{\pi}^\star(x_t) \rangle - \langle \bm{L}_t^\star(x_t), \bm{\pi}_t(x_t) \rangle \leq  \sqrt{KU_T}\frac{z_t}{\sqrt{\sum_{\tau=1}^{t} z_\tau}} + \nonumber \\
	 && \sqrt{\frac{K \sum_{\tau=1}^{t} z_\tau}{U_T}} \bigg(\mathbb{E}_{a \sim \bm{\pi}_t(x_t)} (f^\star(x_t,a)-\hat{f}_t(x_t,a))^2 + \mathbb{E}_{a \sim \bm{\pi}_t(x_t)} (g^\star(x_t,a)-\hat{g}_t(x_t,a))^2 \bigg),
\end{eqnarray}
where, in the second term, we have used the fact that $z_t \geq 1.$ 
Summing up the above inequalities for $1 \leq t \leq t',$ it follows that
\begin{eqnarray*}
		\sum_{t=1}^{t'} \langle \bm{L}_t^\star(x_t), \bm{\pi}^\star(x_t) \rangle - \langle \bm{L}_t^\star(x_t), \bm{\pi}_t(x_t) \rangle  \leq 4 \sqrt{KU_T} \sqrt{\sum_{t=1}^{t'} z_t},
\end{eqnarray*}
where we have used the guarantees \eqref{oracle-guarantee} of the online regression oracle $\mathcal{O}_{\textrm{sq}}$ and the fact that for any non-negative sequence $\{z_t\}_{t \geq 1},$ we have 
\begin{eqnarray*}
	\sum_{t=1}^{t'} \frac{z_t}{\sqrt{\sum_{\tau=1}^{t'} z_\tau}} \leq \sum_{t=1}^{T}\int_{\sum_{\tau=1}^{t-1} z_\tau}^{\sum_{\tau=1}^t z_\tau} \frac{dx}{\sqrt{x}} = \int_{0}^{\sum_{t=1}^{t'} z_t} \frac{dx}{\sqrt{x}} = 2 \sqrt{\sum_{t=1}^{t'} z_t}.
\end{eqnarray*}
Finally, using the fact that $z_t \leq 1 + \Phi'(Q(t-1))^2,$ we obtain our final bound for the regret for learning the surrogate reward functions:
\begin{eqnarray} \label{final-reg-decomp}
	\textrm{Regret}_t' \equiv \sum_{\tau=1}^t \langle \bm{L}_\tau^\star(x_\tau), \bm{\pi}^\star(x_\tau) \rangle -  \langle \bm{L}_\tau^\star(x_\tau), \bm{\pi}_\tau(x_\tau) \rangle  \leq 4 \sqrt{KU_Tt } + 4 \sqrt{KU_T} \sqrt{\sum_{\tau=1}^{t-1} \Phi'(Q(\tau))^2}. 
\end{eqnarray}
Finally, taking the expectation of both sides of \eqref{reg-decomp-ineq1}, summing them up and using \eqref{final-reg-decomp} for bounding the RHS of the inequality, we conclude the fundamental \textbf{Regret Decomposition Inequality}, which will be the starting point of further analysis:
\begin{align} \label{reg-decomp-ultimate}
	&\mathbb{E}(\Phi(Q(t))) - \mathbb{E}(\Phi(Q(0))) + \mathbb{E} \textrm{Regret}_t (\pi^\star) \nonumber \\ &\leq 4 \sqrt{KU_Tt } + \sum_{\tau=1}^t \mathbb{E}\Phi''\big[(Q(\tau))]\big)+ 4 \sqrt{KU_T}\mathbb{E} \sqrt{\sum_{\tau=1}^{t-1}\bigg([\Phi'(Q(\tau))]^2\bigg)}.
\end{align}

%% file: results.tex
\section{Performance Bounds for Algorithm \ref{ccb2}}
In this section, we present the main theorem of the paper, which provides the \Regret~and \CCV~bounds for the proposed \CCB~policy across different benchmark classes. 

\begin{theorem}[Performance Bounds for Constrained Contextual Bandit]
\label{thm:main}
Under the realizability assumption (Assumption 1) and given an online regression oracle $\mathcal{O}_{sq}$ with cumulative squared error $U_T$ (given by Eqn.\ \eqref{oracle-guarantee}), Algorithm \ref{ccb2} achieves the following expected Regret and Cumulative Constraint Violation (\CCV) bounds:

\begin{enumerate}[label=(\alph*), leftmargin=2em]
    \item \textbf{Feasible in Expectation:} If the benchmark policy $\pi^\star$ is feasible in expectation (Def. \ref{assum:in_expect}), then choosing $\Phi(x)=\nicefrac{x^2}{V}$ with $V = \sqrt{KTU_T}$,  Algorithm \ref{ccb2} achieves:
    \begin{align*}
        \mathbb{E}\mathsf{Regret}_T = \mathcal{O}(\sqrt{K}T^{3/4}U_T^{1/4}), ~~
        \mathbb{E} \mathsf{CCV}_T = \mathcal{O}(\sqrt{K}T^{3/4}U_T^{1/4}). 
    \end{align*}

    \item \textbf{Slater's Condition:} If the benchmark $\pi^\star$ additionally satisfies Slater's condition with parameter $\epsilon > 0$ (Def. \ref{assum:in_expect-slater}), then with the same Lyapunov function $\Phi(\cdot),$ the average \CCV in part (a) can be improved to:
    \begin{equation*}
        \frac{1}{T}\sum_{\tau=1}^{T} \mathbb{E}\mathsf{CCV}_\tau = \mathcal{O}\left(\frac{\sqrt{KTU_T}}{\epsilon}\right).
    \end{equation*}
\begin{corollary}
An application of the Markov inequality shows that for any fixed, say $99\%$ of the total number of rounds, the \CCV~is at most $O(\frac{\sqrt{KTU_T}}{\epsilon}).$
\end{corollary}
    \item \textbf{Almost Surely Feasible:} If $\pi^\star$ is almost surely feasible (Def. \ref{assum:almost_sure}), then choosing $\Phi(x) = \exp(\lambda x)$ with $\lambda= (8 \sqrt{KU_T T})^{-1},$ Algorithm \ref{ccb2} yields:
    \begin{align*}
         \mathbb{E}\mathsf{Regret}_T \le O(\sqrt{KU_T T}), ~~
         \mathbb{E}\mathsf{CCV}_T = \tilde{\mathcal{O}}(\sqrt{KTU_T}).
    \end{align*}

    \item \textbf{Long-term Budget Feasible with non-negative cost (the $\mathsf{CBwK}$ problem):} For a benchmark $\pi^\star$ that is long-term budget feasible for a total budget $B_T \ge 0$ (Def. \ref{assum:lt_budget}) and non-negative costs, then choosing $\Phi(x) = \exp(\lambda x)$ with $\lambda= (8 \sqrt{KU_T T} + 2B_T)^{-1},$ Algorithm \ref{ccb2} yields:
    \begin{align*}
         \mathbb{E}\mathsf{Regret}_T \le O(\sqrt{KU_T T}), ~~
         \mathbb{E} \mathsf{CCV}_T = \tilde{\mathcal{O}}(\sqrt{KTU_T} + B_T \log T).
    \end{align*}

    \item \textbf{Non-negative average \Regret:} If the benchmark policy $\pi^\star$ is feasible in expectation and the online policy has non-negative average regret, \emph{i.e.,} $\frac{1}{T}\sum_{t=1}^T \mathbb{E} \mathsf{Regret}_t(\pi^\star) \geq 0,$ then choosing $\Phi(x)=\nicefrac{x^2}{V}$ with $V = \sqrt{KTU_T} + B_T$,  Algorithm \ref{ccb2} yields:
\begin{align*}
    \mathbb{E} \mathsf{Regret}_T = O(\sqrt{KU_T T}), ~~
    \mathbb{E}\mathsf{CCV}_T = O(\sqrt{KU_T T}).
\end{align*}
\item \textbf{Long-term Budget Feasible with signed cost (the $\mathsf{CBwLC}$ problem):} Assuming stochastic contexts and for a benchmark $\pi^\star$ that is long-term budget feasible for a total budget $B_T$ (Def. \ref{assum:lt_budget}), where $B_T$ is small ($B_T=o(T)$), and signed costs (allowing negative values, c.f. part (d)), then choosing $\Phi(x)=\nicefrac{x^2}{V}$ with $V = \sqrt{KTU_T} + B_T$,  Algorithm \ref{ccb2} yields:
\begin{align*}
    \mathbb{E} \mathsf{Regret}_T = O(\sqrt{K}T^{3/4}U_T^{1/4}+B_T), ~~
    \mathbb{E}\mathsf{CCV}_T = O(\sqrt{K}T^{3/4}U_T^{1/4}+\sqrt{TB_T}).
\end{align*}
\end{enumerate}
\end{theorem}
Due to space constraints, we only give proof of part (a) of the above theorem in the main paper. The complete proofs of the remaining results are provided in Appendix B-F. Thanks to the regret decomposition framework (Eqn.\ \eqref{reg-decomp-ultimate}), the proofs are elementary, requiring just a few lines of algebra. 
\paragraph{Proof of Theorem \ref{thm:main} (a) [Feasible in Expectation Benchmark (Definition \ref{assum:in_expect})]} \label{sec:in_expect}

Let us choose the Lyapunov function $\Phi(x) = \frac{x^2}{V}$, for some parameter $V>0$ which will be fixed later. Then, multiplying both sides by $V,$ the regret decomposition inequality in Eqn.\ \eqref{reg-decomp-ultimate} yields for any $t \in [T]:$
\begin{align}
\label{eq:original_ineq}
\mathbb{E}Q^2(t) - Q^2(0) + V\mathbb{E} \textrm{Regret}_t (\pi^\star) \leq 4V\sqrt{K U_Tt} + 2t + 8 \sqrt{K U_T}\sqrt{\sum_{\tau=1}^t \mathbb{E}Q^2(t)}.
\end{align}
Summing from $t=1$ to $T$, and noting that $Q(0)=0,$ we have
\begin{align} \label{Q-decomp-2}
    \sum_{t=1}^T\mathbb{E}Q^2(t) + V\sum_{t=1}^T\mathbb{E} \textrm{Regret}_t (\pi^\star) \leq 4VT^{3/2}\sqrt{K U_T} + 2T^2 + 8 T\sqrt{K U_T}\sqrt{\sum_{t=1}^T \mathbb{E}Q^2(t)}.
\end{align}

Let us now define the variable $R(T) := \sqrt{\sum_{t=1}^T \mathbb{E}Q^2(t)}.$ Note that we trivially have $\mathbb{E} \textrm{Regret}_t (\pi^\star) \geq -2T$. This is because $\mathbb{E}\textrm{Regret}_t (\pi^\star) = \sum_\tau f^\star(\pi_{\tau}(x_\tau),x_\tau) - f^\star(\pi^\star(x_\tau),x_\tau)$ and since we assume that the function $f^\star$ takes values in $[-1,1],$ we have $f^\star(\pi_{\tau}(x_\tau),x_\tau) - f^\star(\pi^\star(x_\tau),x_\tau) \geq -2\,\, \forall \tau \in [T]$. Plugging in this bound in inequality \eqref{Q-decomp-2}, we conclude:
\begin{align*}
    R^2(T) \leq VT^2 + 2T^2 + 4VT^{3/2}\sqrt{K U_T} + 8 T\sqrt{K U_T}R(T).
\end{align*}
Noticing that the above inequality is of the form $x^2 \leq ax + b$ where $x\equiv R(T),$ and using the bound from Lemma \ref{lemma:quadratic}, we obtain the following upper bound on $R(T)$:
\begin{align}
\label{eq:bound-R(t)}
    R(T) \leq \sqrt{V}T + 2T + 2K^{1/4}\sqrt{V}T^{3/4} (U_T)^{1/4} + 8 T\sqrt{K U_T}.
\end{align}

We further note that inequality \eqref{eq:original_ineq} can be rewritten in term of $R(T)$ as follows:
\begin{align} \label{q-reg-ineq0}
\mathbb{E}Q^2(T) + V\mathbb{E} \textrm{Regret}_T (\pi^\star) \leq 4V\sqrt{K U_TT} + 2T + 8 \sqrt{K U_T}R(T).
\end{align}
Plugging in the upper bound on $R(T)$ from \eqref{eq:bound-R(t)} into the above inequality, we obtain
\begin{align} \label{Q-reg-ineq1}
&\mathbb{E}Q^2(T) + V\mathbb{E} \textrm{Regret}_T (\pi^\star) \nonumber \\ \leq &4V\sqrt{K U_TT} + 2T + 8 \sqrt{VK U_T}T + 16\sqrt{KU_T}T + 16 K^{3/4}\sqrt{V} T^{3/4} U_T^{3/4} + 64KU_TT.
\end{align}
Hence, using $\mathbb{E}(Q^2(T))\geq 0,$ we obtain the following regret bound:
\begin{align*}
    \mathbb{E} \mathsf{Regret}_T \leq O\bigg(\max(\sqrt{K TU_T}, \sqrt{\frac{KU_T}{V}}T, \frac{1}{\sqrt{V}}K^{3/4} T^{3/4} U_T^{3/4}, \frac{KU_T T}{V})\bigg).
\end{align*}
Furthermore, substituting the trivial regret lower bound $\mathbb{E} \textrm{Regret}_t (\pi^\star) \geq -2T$ into \eqref{Q-reg-ineq1} and using Jensen's inequality $\mathbb{E}Q^2(T) \geq (\mathbb{E}(Q(T)))^2,$ we obtain the following \CCV~bound
\begin{eqnarray} \label{q-bd-no-assump}
    \mathbb{E}Q(T)\leq O\bigg(\max(\sqrt{VT},\sqrt{V}(K TU_T)^{1/4},  (VK U_T)^{1/4}\sqrt{T} , K^{3/8}V^{1/4} T^{3/8} U_T^{3/8}, \sqrt{KU_TT})\bigg).
\end{eqnarray}
Choosing $V = \sqrt{KTU_T},$ we conclude $ \mathbb{E} \mathsf{Regret}_T = O(\sqrt{K}T^{3/4}U_T^{1/4}),~ \mathbb{E}\mathsf{CCV}_T = O(\sqrt{K}T^{3/4}U_T^{1/4}).$

%% file: conclusion.tex
\section{Conclusion}
We propose a modular, unified algorithmic framework for constrained contextual bandits under general realizability assumptions, with adversarially chosen contexts. By removing any distributional assumptions on the context sequence, our results apply to non-stationary environments and automatically subsume the stochastic setting as a special case. The central technical contribution is a general regret decomposition inequality that cleanly separates the roles of exploration (via Inverse Gap Weighting), constraint management (via Lyapunov-based surrogates), and statistical estimation (via online regression oracles). This decomposition yields a transparent analysis pipeline through which regret and cumulative constraint violation guarantees follow immediately for different feasibility benchmarks and structural assumptions considered in the prior literature, including almost sure feasibility, Slater’s condition, and constrained contextual bandits with knapsacks.

%% file: budget_constraints.tex

%% file: CBwK.tex
\section{Related Works}
\paragraph{Contextual Bandits:}
The Multi-armed Bandit (MAB) problem is a fundamental paradigm in online decision-making \citep{auer2002finite,bubeck2012regret,lattimore2020bandit}, where an agent seeks to maximize cumulative rewards in the face of environmental uncertainty. Contextual bandits (CB) extend this by providing the learner with context to inform their actions. Much of the literature on CB assumes linear realizability \citep{rusmevichientong2010linearly,abbasi2011improved,li2010contextual,abeille2017linear,agrawal2013further}, employing classic strategies like Upper Confidence Bound (UCB) \citep{abbasi2011improved}, Thompson sampling \citep{chapelle2011empirical}, and randomized exploration \citep{vaswani2019old} to create efficient algorithms.

To move beyond strict realizability, researchers explored \CB~using classification oracles \citep{dudik2011efficient,agarwal2014taming, langford2007epoch}. While these allow for more general models, the underlying classification task can be computationally intractable even for simple hypothesis classes \citep{klivans2009cryptographic}. In response, \CB~frameworks based on regression oracles have emerged as a more practical and computationally efficient alternative \citep{foster2018practical,foster2020beyond,simchi2022bypassing}.

\paragraph{The Decision-to-Estimation framework:} Building upon the concept of using regression for exploration \cite{foster2018practical} introduced a method to handle general function classes via an offline regression oracle. This approach is particularly appealing because weighted least squares oracles are widely available, efficient to implement, and support gradient-based optimization. Empirical studies in \cite{foster2018practical, bietti2021contextual} have demonstrated that these algorithms often outperform existing \CB~methods in practice.

However, the algorithm in \cite{foster2018practical} remains theoretically sub-optimal, potentially incurring linear regret in the worst case. This led to the open question of whether optimal regret could be achieved using an offline-oracle-based approach. More recently, \cite{foster2020beyond} achieved optimal regret by utilizing an online regression oracle. While this reduction is a significant theoretical advancement that holds under minimal assumptions and adversarial settings, the requirement for an online oracle can be somewhat restrictive, as efficient algorithms for online regression are known for only a limited set of function classes, compared to the much broader range available for offline regression.

\paragraph{Contextual Bandits with Long-term Constraints:}

The field of contextual bandits with constraints expands the traditional framework by introducing costs, where the goal of the learner is to maximize the cumulative rewards while adhering to long-term constraints. Early investigations primarily addressed knapsack constraints, where the process ends immediately if a resource is depleted in what is known as ``hard stopping"\citep{badanidiyuru2014resourceful, agrawal2014bandits, wu2015algorithms, agrawal2016linear, badanidiyuru2018bandits, sivakumar2022smoothed, chzhen2023small, guo2025stochastic}. For the linear case, \citep{agrawal2016linear} achieved an optimal theoretical regret of $\tilde{O}((1+\frac{\nu^{*}}{b})\sqrt{T})$. These findings were later generalized to broader function classes using regression oracles for concurrent reward and cost estimation \citep{han2023optimal, slivkins2022efficient}. Works which assume hard-stopping have to inevitably rely on the existence of a NULL arm, which is defined to be an
arm with zero cost and consumption at all rounds. 

Fairness has also emerged as a significant area of study in the bandit literature. Key concepts include individual fairness—treating similar entities alike \citep{dwork2012fairness, joseph2016fairness, chzhen2023small} —and minimum-selection fairness, which mandates that every arm is chosen at least a certain percentage of the time \citep{li2019combinatorial, chen2020fair, claure2020multi, sinha2023banditq}. Additionally, group fairness research seeks to equalize average expenditures across different demographic groups \citep{chohlas2024learning, chzhen2023small}. These specialized algorithms embed fairness directly into the policy-making process.

More generalized settings, such as those in \citet{slivkins2023contextual}, evaluate performance via both regret and long-term cumulative violation, allowing for both positive and negative costs. While they established a $\tilde{O}(\sqrt{T}/\delta)$ guarantee, their approach requires Slater's condition and prior knowledge of the Slater constant $\delta$. \citet{guo2024stochastic} provided the first results without Slater's condition in the adversarial setting, achieving $\tilde{O}(T^{\frac{3}{4}})$ regret and violation, or $\tilde{O}(\sqrt{T}/\delta^{2})$ when the condition holds without needing to know $\delta$ beforehand. In non-contextual settings, \citet{sinha2023banditq} also achieved $\tilde{O}(T^{\frac{3}{4}})$ without Slater's condition. In the non-contextual adversarial setting, \citet{bernasconi2024beyond} achieved $O(\sqrt{T})$ $\nicefrac{\delta}{(1+\delta)}$-approximate regret and \CCV~in the adversarial setting, where the benchmark could become arbitrarily weaker as Slater's constant $\delta >0$ becomes small. The challenge of achieving $\tilde{O}(\sqrt{T})$ regret and violation through an efficient algorithm for general contextual bandits without Slater's condition remains an open question that this paper addresses. Other related studies \citep{moradipari2021safe, amani2019linear, khezeli2020safe, pacchiano2025contextual, pmlr-v247-gangrade24a, pmlr-v162-chen22e, gangrade2025constrainedlinearthompsonsampling} focus on stage-wise safety, but these typically require a known feasible starting point or the repetitive calculation of safe regions. Below, we separately discuss the works that are closest to ours.

\subsection{Comparison with Related Works} \label{comparison}
We summarize our results in the table below and place them in context with related work
\input{related_works}

\citet{guo2024stochastic} uses a Lyapunov-based framework to achieve regret and violation bounds of $\tilde{O}(T^{3/4}U^{1/4})$ without assuming Slater's condition and under the assumption of expected feasibility of the benchmark. We get the same bounds under the same benchmark, but we tackle the much more challenging scenario of \emph{adversarial} contexts. 

When the Slaters' condition is assumed, \citet{guo2024stochastic} achieve regret and violation bounds of $\tilde{O}(\frac{\sqrt{TU}}{\epsilon^2})$ which holds when $\epsilon = \Omega\bigg(\sqrt{\frac{U}{T}}\bigg)$. In our case, we obtain an improvement of a factor of $\frac{1}{\epsilon}$ to get a bound of $\tilde{O}(\frac{\sqrt{TU}}{\epsilon})$ for the \textit{average} \CCV, partially matching a conjecture raised by \citet[Section 4.2]{guo2024stochastic} in the stronger adversarial setting. 

In a more recent work, \cite{guotriple} achieves $\Tilde{O}(\sqrt{TU})$ bounds on regret and violation, but does so under the very strong assumption of a known probability distribution of contexts. Needless to say, in most practical cases, this assumption is not satisfied. On the other hand, under the much milder assumption of almost sure feasibility and adversarial contexts, we obtain performance bounds of the same order of magnitude. We also achieve the same bound when we assume a non-trivial lower bound on the average \Regret~ and expected feasibility.

Our results for the long-term budget-feasible benchmark align closely with the literature on Contextual Bandits with Knapsacks ($\mathsf{CBwK}$) when costs are non-negative, and on Contextual Bandits with Linear Constraints ($\mathsf{CBwLC}$) when negative costs are permitted. 
\cite{han2023optimal} considered the problem of stochastic contextual bandits with knapsack constraints. However, they assume the hard-stopping setting, so their results are not directly comparable to ours. \citet{guo2025stochastic} also considered the hard-stopping setting and improved upon the results presented in \citet{han2023optimal}. \citet{slivkins2023contextual} considered the continuing case for $\mathsf{CBwLC}$ where the learner continues to use resources even after the budget is exhausted, but they assumed a very large budget $\Omega(T)$ and also assumed strict feasibility in the form of a known Slater's condition. They achieved regret and violation bounds of $O(T^{3/4}U_T^{1/4})$. 

\section{Improved bounds assuming Slater's condition} \label{slater-sec}

When we assume Slater's condition holds (but the value of Slater's constant need not be known to the algorithm), we can strengthen the Regret decomposition inequality derived above. Since, in this case $\mathbb{E}_{a^\star \sim \pi_0}g_t(x_t, a^\star) \leq -\epsilon,$ taking the conditional expectation of both sides of \eqref{drift-ineq1} w.r.t. $\mathcal{F}_{t-1}$, we get one additional negative term on the RHS as shown below:
\begin{eqnarray}\label{RDI-Slater}
	&&\mathbb{E}(\Phi(Q(t))|\mathcal{F}_{t-1}) - \Phi(Q(t-1))+ \langle \bm{f}^\star(x_t), \bm{\pi}_t(x_t) - \bm{\pi_0}(x_t)\rangle \nonumber\\
	&\leq& \langle \bm{L}^\star_t(x_t), \bm{\pi}_t(x_t) - \bm{\pi}_0(x_t) \rangle
 	 + \frac{1}{2}(\Phi''(Q(t))+ \Phi''(Q(t-1))) - \epsilon \Phi'(Q(t-1)),
\end{eqnarray} 
Similar to \eqref{reg-decomp-ultimate}, taking expectations of both sides, summing them up and substituting the upper bound for the surrogate regret, we have the following inequality
\begin{align*}
	&\mathbb{E}(\Phi(Q(t))) - \mathbb{E}(\Phi(Q(0))) + \mathbb{E} \textrm{Regret}_t (\bm{\pi}_0) \\ \leq &4 \sqrt{KU_Tt } + \sum_{\tau=1}^t \mathbb{E}\Phi''\big[(Q(\tau))]\big)+ 4 \sqrt{KU_T} \sqrt{\sum_{\tau=1}^{t-1} \mathbb{E}\bigg([\Phi'(Q(\tau))]^2\bigg)}
	 -\epsilon \sum_{\tau=1}^{t-1} \mathbb{E}\Phi'(Q(\tau)). 
\end{align*}

As before, we choose the quadratic Lyapunov function $\Phi(x) = \frac{x^2}{V},$ with $V=\sqrt{KTU_T}$.
 Proceeding as before, we have
\begin{eqnarray}\label{avg-q-bd}
	\mathbb{E}Q^2(T) + 2\epsilon \sum_{\tau=1}^{T-1} \mathbb{E}Q(\tau) \leq VT+ 2T + 4 \sqrt{KU_TT} + 8 \sqrt{KU_T} R(T).
\end{eqnarray} 
From our previous results in Section \ref{assum:in_expect}, we concluded that $R(T) = O((KU_T)^{1/4} T^{5/4}).$ Hence, from Eqn.\ \eqref{avg-q-bd}, we conclude that 
\begin{eqnarray*}
	\frac{1}{T} \sum_{\tau=1}^T \mathbb{E}Q(\tau) = O(\frac{V}{\epsilon}) = O\big(\frac{\sqrt{KTU_T}}{\epsilon}\big),
\end{eqnarray*}
which gives a sharper bound than the bound in Section \ref{assum:in_expect} for the average $\mathsf{CCV},$ improved by a factor of $O(\frac{1}{\epsilon})$. Using Markov's inequality, this result shows that for any fixed, say $99\%$ of the total number of rounds, the \CCV~is at most $O(\frac{\sqrt{KTU_T}}{\epsilon}).$ This result also provides a valuable insight into an open question posed by \citep[Lemma 6]{guo2024stochastic} in the stronger adversarial setting, who conjectured that the above result holds for the terminal \CCV. 

\section{Almost Surely Feasible Benchmark (Definition \ref{assum:almost_sure})} \label{as-analysis}

In case of almost sure feasibility, we first modify the problem instance where each cost function is replaced with its positive part, \emph{i.e.,} $g_t (\cdot, \cdot) \gets \max(0, g_t(\cdot, \cdot))), \forall t.$ Because of the almost sure feasibility assumption, it follows that $\pi^\star$ is also a feasible policy for the new problem instance, and hence, the Regret Decomposition inequality \eqref{reg-decomp-ultimate} remains valid. Furthermore, using the non-negative property of the cost functions, it follows that the \CCV~sequence $\{Q(t)\}_{t \geq 1}$ is almost surely monotone non-decreasing - a new property which does not hold in the previous cases. Finally, we choose the Lyapunov function to be the exponential function, \emph{i.e.,} $\Phi(x)\equiv \exp(\lambda x),$ for some $\lambda >0$ to be fixed later. With this choice, \eqref{reg-decomp-ultimate} simplifies to
\begin{eqnarray*}
	\mathbb{E}\exp(\lambda Q(T))- 1 + \mathbb{E} \textrm{Regret}_T (\pi^\star) &\leq& 4 \sqrt{KU_TT} + \lambda^2 T \mathbb{E}\exp(\lambda Q(T)) \\
    &&+ 4 \sqrt{KU_TT} \lambda \mathbb{E}\exp(\lambda Q(T)). 
\end{eqnarray*}
Finally, choosing $\lambda = \frac{1}{8\sqrt{KU_T T}},$ it follows that 
 \begin{eqnarray}\label{final-bd2}
 	\mathbb{E}\exp(\lambda Q(T))- 1 + \mathbb{E} \textrm{Regret}_T (\pi^\star) \leq 4 \sqrt{KU_TT} + \frac{2}{3} \mathbb{E}\exp(\lambda Q(T)).
 \end{eqnarray}
 which implies
 \begin{eqnarray*}
 	\frac{1}{3}\mathbb{E}\exp(\lambda Q(T))- 1 + \mathbb{E} \textrm{Regret}_T (\pi^\star) \leq 4 \sqrt{KU_TT}.
 \end{eqnarray*}
 Since $Q(T) \geq 0,$ the above inequality immediately yields \[\mathbb{E} \textrm{Regret}_T (\pi^\star) \leq  4 \sqrt{KU_TT}+ \nicefrac{2}{3}.\]
 Finally, to bound the constraint violations (\textsf{CCV}), note that since all cost vectors are upper bounded by unity, we have $\textrm{Regret}_T (\pi^\star) \geq -T.$ Substituting this lower bound in \eqref{final-bd2}, we obtain 
 \begin{eqnarray*}
 \frac{1}{3}\exp(\lambda \mathbb{E} Q(T)) \stackrel{\textrm{(Jensen's ineq.)}}{\leq}\frac{1}{3}\mathbb{E}\exp(\lambda Q(T)) \leq 1+ T + 4 \sqrt{KU_TT},
 \end{eqnarray*}
which implies the following bound for \CCV:
 \begin{eqnarray*}
 \mathbb{E}Q(T) = \tilde{O}(\sqrt{KTU_T}).
 \end{eqnarray*}
\section{Contextual Bandits with Knapsacks Constraints (\cbwk)}
\label{sec:lt_budget}
In this section, we consider non-negative costs (resource consumption) and a long-term budget feasible benchmark (Definition \ref{assum:lt_budget}) with an \emph{arbitrary} budget of $B_T \geq 0.$

For simplicity, we consider a single resource. Our derivation also generalizes to multiple resources as discussed in Section \ref{subsec:multiple_resources}. This particular problem is known as the constrained contextual bandits with knapsack constraints (\cbwk) in the literature. \cbwk~was considered earlier by \citet{slivkins2023contextual, han2023optimal} in the special case of a large budget regime where $B_T=\Omega(T)$ and assuming a known and positive slack to the resource constraint. Their algorithm is based on a primal-dual scheme, called $\mathsf{LagrangeBwK}$, first introduced for the (non-contextual) Bandits with Knapsacks ($\mathsf{BwK}$) problem \citep{badanidiyuru2018bandits}. Our method is entirely different from $\mathsf{LagrangeBwK},$ and uses the previous regret decomposition scheme with an exponential Lyapunov function as described next. 
 
We begin with inequality \eqref{drift-ineq1} that gives an upper bound to the sum of the drift and incremental regret. Choosing the benchmark policy $\pi^\star$ to be any long-term budget feasible policy, and taking the conditional expectation of both sides of Eqn.\ \eqref{drift-ineq1} with respect to the randomness of the reward and cost functions and the randomness of the online and the benchmark policies, it follows that
\begin{align}
\label{eq:with_benchmark}
 	&\mathbb{E}(\Phi(Q(t))|\mathcal{F}_{t-1}) - \Phi(Q(t-1))+ \langle \bm{f}^\star(x_t), \bm{\pi^\star}(x_t) - \bm{\pi}_t(x_t)\rangle \nonumber \\
    \leq &\langle \bm{L}^\star_t(x_t),\bm{\pi}^\star(x_t) -\bm{\pi}_t(x_t) \rangle 
 	 + \frac{1}{2}(\Phi''(Q(t))+ \Phi''(Q(t-1))) + \Phi'(Q(t-1)) \mathbb{E}_{a^\star \sim \pi^\star(x_t)} g_t(x_t, a^\star),
 \end{align} 
where the target surrogate function $L_t^\star$ and the estimated surrogate function $\hat{L}_t$ have been defined in Eqns.\ \eqref{est-surr-cost} and \eqref{surr-cost-def} respectively. Next, following exactly the same derivation as in Section \ref{surr-reg-bd}, we conclude the following generalized form of regret decomposition inequality
 \begin{eqnarray}  \label{new-reg-decomp}
	&&\mathbb{E}(\Phi(Q(t))) - \mathbb{E}(\Phi(Q(0))) + \mathbb{E} \textrm{Regret}_t (\pi^\star)  \leq 4 \sqrt{KU_Tt }\nonumber \\ &&+ \sum_{\tau=1}^t \mathbb{E}\Phi''\big[(Q(\tau))]\big)+ 4 \sqrt{KU_T}\mathbb{E} \sqrt{\sum_{\tau=1}^{t-1}\bigg([\Phi'(Q(\tau))]^2\bigg)} + \mathbb{E}[\Phi'(Q(t-1))] B_T.
\end{eqnarray}
While bounding the last term, we have used the fact that since the costs are non-negative, the sequence $\{\Phi'(Q(\tau))\}$ is non-decreasing and hence, we have almost surely 
\begin{eqnarray} \label{new-reg-decomp3}
 \sum_{\tau=1}^t \Phi'(Q(\tau-1)) \mathbb{E}_{a^\star \sim \pi^\star(x_\tau)} c_\tau(x_\tau, a^\star) &\leq& \Phi'(Q(t-1)) \sum_{\tau=1}^t \mathbb{E}_{a^\star \sim \pi^\star(x_\tau)} c_\tau(x_\tau, a^\star) \nonumber \\
 &\stackrel{(a)}{\leq}&  \Phi'(Q(t-1)) B_T, 
 \end{eqnarray}
where (a) follows from the long-term budget-feasibility of the benchmark policy $\pi^\star.$
Note that the only difference between Eqn.\ \eqref{new-reg-decomp} and the previous regret decomposition inequality \eqref{reg-decomp-ultimate} is the presence of the term involving budget $B_T$ in the former. Because of this formal similarity, the analysis follows a similar line to that in Section \ref{as-analysis}. 

Using the monotonicity of the sequence $\{Q(\tau)\}_{\tau}$ once again and choosing the Lyapunov function to be the exponential function, \emph{i.e.,} $\Phi(x)\equiv \exp(\lambda x),$ for some $\lambda >0$ (to be fixed later), inequality \eqref{new-reg-decomp} simplifies to
\begin{eqnarray*}
	\mathbb{E}\exp(\lambda Q(T))- 1 + \mathbb{E} \textrm{Regret}_T (\pi^\star) &\leq& 4 \sqrt{KU_TT} + \lambda^2 T \mathbb{E}\exp(\lambda Q(T)) \\
    &&+ \lambda(4 \sqrt{KU_TT}+B_T)\mathbb{E}\exp(\lambda Q(T)). 
\end{eqnarray*}
Finally, choosing $\lambda = (8\sqrt{KU_T T} + 2B_T)^{-1},$
 we conclude
 \begin{eqnarray}\label{final-bd2-budget}
 	\mathbb{E}\exp(\lambda Q(T))- 1 + \mathbb{E} \textrm{Regret}_T (\pi^\star) \leq 4 \sqrt{KU_TT} + \frac{2}{3} \mathbb{E}\exp(\lambda Q(T)).
 \end{eqnarray}
 which yields
 \begin{eqnarray*}
 	\frac{1}{3}\mathbb{E}\exp(\lambda Q(T))- 1 + \mathbb{E} \textrm{Regret}_T (\pi^\star) \leq 4 \sqrt{KU_TT}.
 \end{eqnarray*}
 Since $Q(T) \geq 0,$ the above inequality immediately implies the following regret bound \[\mathbb{E} \textrm{Regret}_T (\pi^\star) \leq  4 \sqrt{KU_TT}+ \nicefrac{2}{3}.\]
 Finally, to bound the constraint violations (\textsf{CCV}), note that since all cost vectors are upper bounded by unity, we have $\textrm{Regret}_T (\pi^\star) \geq -T$. Substituting this in \eqref{final-bd2-budget}, it follows that 
 \begin{eqnarray*}
 \frac{1}{3}\exp(\lambda \mathbb{E} Q(T)) \stackrel{\textrm{(Jensen's ineq.)}}{\leq}\frac{1}{3}\mathbb{E}\exp(\lambda Q(T)) \leq 1+ T + 4 \sqrt{KU_TT},
 \end{eqnarray*}
 which implies the following bound for the \CCV:
 \begin{eqnarray*}
 \mathbb{E}Q(T) = \tilde{O}(\sqrt{KTU_T}) + O(B_T\log T).
 \end{eqnarray*}
 \paragraph{Discussion:}
Our results improve upon previous results on \cbwk~on multiple fronts. While \cite{slivkins2023contextual} assume 
(1) Stochastic contexts
(2) A positive and known slack $\zeta$ to the resource constraints \cite[Theorem 3.6]{slivkins2023contextual} and (3) Large budget regime where $B_T = \Omega(T)$, we remove all of the above rather restrictive assumptions by considering (1) adversarial contexts, (2) no assumption on the slack, and (3) arbitrary budgets with a compact and transparent analysis, directly leveraging the seminal \SqCB~framework. 
\subsection{Extension to Multiple Resources}
\label{subsec:multiple_resources}
To enable the analysis with $m$ resources, we would need to define multiple virtual queues $Q_i$ for each resource $i$ and a new surrogate reward function that accounts for all resources.
\begin{align}
    Q_i(t) = Q_i(t-1) + g_{t,i}(x_t,a_t).
\end{align}
We also define a new form of the surrogate reward function
\begin{eqnarray} \label{est-surr-cost-multiple}
	L^\star_t(x_t, a) = f^\star(x_t,a) - \sum_{i=1}^m\Phi'(Q_i(t-1)) g_i^\star(x_t,a), ~~ a \in [K], 
\end{eqnarray}
and the \emph{estimated surrogate} function as:
\begin{eqnarray} \label{surr-cost-def-multiple}
	\hat{L}_t(x_t, a) = \hat{f}_t(x_t,a) - \sum_{i=1}^m \Phi'(Q_i(t-1)) \hat{g}_{t,i}(x_t,a), ~~ a \in [K]. 
\end{eqnarray}
The rest of the analysis would be similar to the one in a single resource. 

%% file: related_works.tex
\begin{table*}[!h]
\centering
\label{Table:Related Works}
\renewcommand{\arraystretch}{1.5} 

\scalebox{0.9}{
\begin{tabular}{lll} 
\toprule
\multicolumn{1}{l}{\textbf{Benchmark}} & \multicolumn{1}{c}{\Regret} & \multicolumn{1}{c}{\CCV} \\ 
\midrule
Feasible in Expectation (Defn. \ref{assum:in_expect}) & $\mathcal{O}(T^{3/4}U_T^{1/4})$ & $\mathcal{O}(T^{3/4}U_T^{1/4})$ \\
Feasible in Expectation with Slater’s condition (Defn. \ref{assum:in_expect-slater}) & $\mathcal{O}(T^{3/4}U_T^{1/4})$ & $\mathcal{O}({\sqrt{T}})$ \\
Almost Surely Feasible (Defn. \ref{assum:almost_sure}) & $\mathcal{O}({\sqrt{TU_T}})$ & $\mathcal{O}({\sqrt{TU_T}})$ \\
Long-term Budget Feasible (Defn. \ref{assum:lt_budget}) ($\mathsf{CBwK}$) & $\mathcal{O}({\sqrt{TU_T}})$ & $\mathcal{O}({\sqrt{TU_T}}+B_T\log T)$ \\
Long-term Budget Feasible (Defn. \ref{assum:lt_budget}) ($\mathsf{CBwLC}$ under stochastic contexts) & $\mathcal{O}(T^{3/4}U_T^{1/4}+B_T)$ & $\mathcal{O}(T^{3/4}U_T^{1/4}+\sqrt{TB_T})$ \\
$\mathsf{Regret}_T \geq 0$ & $\mathcal{O}({\sqrt{TU_T}})$ & $\mathcal{O}({\sqrt{TU_T}})$ \\ 
\bottomrule
\end{tabular}
}
\caption{Summary of performance bounds achieved by Algorithm \ref{ccb2} under various assumptions on the benchmark (see Theorem \ref{thm:main}). Here $U_T$ is the cumulative squared estimation error (Eqn. \ref{oracle-guarantee}) incurred by the predictions produced by the Online Regression Oracle $\mathcal{O}_{\textrm{sq}}$. $B_T$ is the value of long-term budget that must be satisfied by the budget feasible benchmark defined in Definition \ref{assum:lt_budget}. We have also formally stated these results in Theorem \ref{thm:main}.}
\end{table*}